\documentclass[12pt]{l4dc2021} 

\usepackage{mysymbol}

\newtheorem{thm}{Theorem}

\usepackage{tikz,pgfplots}
\usepackage{color} 
\usetikzlibrary{matrix} 
\usetikzlibrary{arrows} 
\usetikzlibrary{calc} 
\usetikzlibrary{shapes}
\pgfplotsset{compat=1.7}
\usetikzlibrary{positioning}
\usetikzlibrary{backgrounds, fit}

\title[Linear Regression with Communication Guarantees]{Linear Regression over Networks with Communication Guarantees}
\usepackage{times}

\author{%
 \Name{Konstantinos Gatsis} \Email{konstantinos.gatsis@eng.ox.ac.uk}\\
 \addr Department of Engineering Science,
 University of Oxford, Parks Road, Oxford, OX1 3PJ
}

\begin{document}

\maketitle

\begin{abstract}%
 {A key functionality of emerging connected autonomous systems such as smart cities, smart transportation systems, and the industrial Internet-of-Things, is the ability to process and learn from data collected at different physical locations. This is increasingly attracting attention under the terms of distributed learning and federated learning. However, in connected autonomous systems, data transfer takes place over communication networks with often limited resources. This paper examines algorithms for communication-efficient learning for linear regression tasks by exploiting the informativeness of the data. The developed algorithms enable a tradeoff between communication and learning with theoretical performance guarantees and efficient practical implementations.}%
\end{abstract}


%



\section{Introduction}

Conventional machine learning approaches require data to be collected at a centralized location to be trained in a centralized manner. However, the emergence of new cyber-physical architectures that are distributed requires rethinking this approach. 
Examples of distributed cyber-physical architectures include the Industrial Internet-of-Things with sensors/actuators/robots connected to access points collecting data to jointly update system models and application operating conditions -- see, for example, Fig.~\ref{fig:architecture}; or future transportation systems with connected vehicles collecting and communicating observations from the road; or large-scale sensing infrastructures in future Smart Cities. As a result, in distributed cyber-physical architectures there is a need to enable learning when data are collected by agents across different physical locations.
%
%

A concept relevant to address this need is federated learning, initiated by Google \citep{konevcny2016federated, bonawitz2019towards}, enabling multiple users to jointly solve a machine learning problem over a communication network from data collected from the users. A major challenge in federated learning is that data can be high-dimensional, making their communication costly and inefficient. To alleviate this communication bottleneck, one direction is based on communicating the machine learning model parameters as they are being trained, such as the weights of a Deep Neural Network, instead of the data itself, or communicate the gradients of the objective with respect to the parameters. In deep learning models with high dimensional weights, sparsification and quantization of the weights or the gradients is further introduced to limit the communication cost~\citep{konevcny2016federated, aji2017sparse, sattler2019sparse, lin2020achieving}. Lazy updates are introduced in \citep{chen2018lag, chen2018communication}, and combinations of non-periodic updates and quantization is explored in \citep{reisizadeh2019fedpaq}.
%
%
Furthermore, when distributed learning is taking place over a wireless network, there is an interest in allocating the available network resources efficiently among the users holding the data~\citep{gunduz2019machine}, such as power~\citep{chen2019joint} or rates~\citep{chang2020communication}. 
The problem of scheduling gradient updates over multiple access channels has also received initial consideration, for example comparing time-based approaches with approaches based on channel conditions~\citep{yang2019scheduling}, or including gradient information~\citep{amiri2020update, chen2020convergence}.


The present paper hinges on the idea that when  model parameters are updated from noisy data, then not all updates are equally informative. Performing updates selectively can be beneficial, and we can evaluate the informativeness of the data by estimating the obtained gain in machine learning performance. Building upon this intuition the proposed algorithms aim for agents to update the machine learning task when their data are most informative, i.e., bring about the most gain. 
By prioritizing updates  with more relevant information, agents can efficiently use communication resources and progress the learning task. 
{This approach builds on recent work by the author, where centralized scheduling of multiple machine learning tasks was explored~\citep{ACC21_Gatsis}, while the present paper addresses the more challenging setup of decentralized communication schemes where agents decide to update independently.}
The technical methodology borrows ideas from the problem of scheduling control tasks over shared communication networks \citep{EisenEtal19a,GatsisEtal15, ayan2019age, soleymani2016optimal}. The methodology is also related to event-triggered learning that tries to update only if necessary~\citep{solowjow2018event, zhao2020event}.


The methodology is developed for the task of solving linear regression problems in Section~\ref{sec:setup} and the communication efficient learning problem is introduced.
Section~\ref{sec:approach} introduces the proposed communication algorithm which prioritizes updates whose data carries the most information, i.e., that would lead to the highest performance gain. 
The approach is theoretically analyzed and  guarantees are provided about both convergence and required communication resources. {Importantly, the proposed approach allows to \textit{provably tradeoff learning performance with communication efficiency}.}
Furthermore, as the method is developed ideally when the data distribution is known, special effort is placed on a practical communication algorithm that uses only the currently available data to estimate how informative the current update will be.
Numerical evaluations in Section~\ref{sec:numerical} validate the performance of the proposed algorithm and show significant performance improvements compared to other approaches in the literature that treat the magnitude of the gradients as a measure of the informativeness of the current update.

\section{Problem Setup}\label{sec:setup}


{
	
	The architecture examined in this paper, shown in Fig.~\ref{fig:architecture}, involves a single access-point/server interested in building a data-driven model by solving a machine learning task on data that are collected by multiple agents. 
	The goal is to find a vector of weights (parameters) $w$ of appropriate dimensions to minimize a performance metric (cost)  $J(w)$.
	The aim will be to achieve this with communication efficiency. 
	This is for example the case when an agent should not communicate all the time over a communication network to update the vector of parameters at the access point/server, e.g., due to capacity constraints. 

}

\begin{figure}[t!]
	{\resizebox{0.5\columnwidth}{!}{\input{multiple_tasks_2}}	}
	\includegraphics[width=0.5\columnwidth]{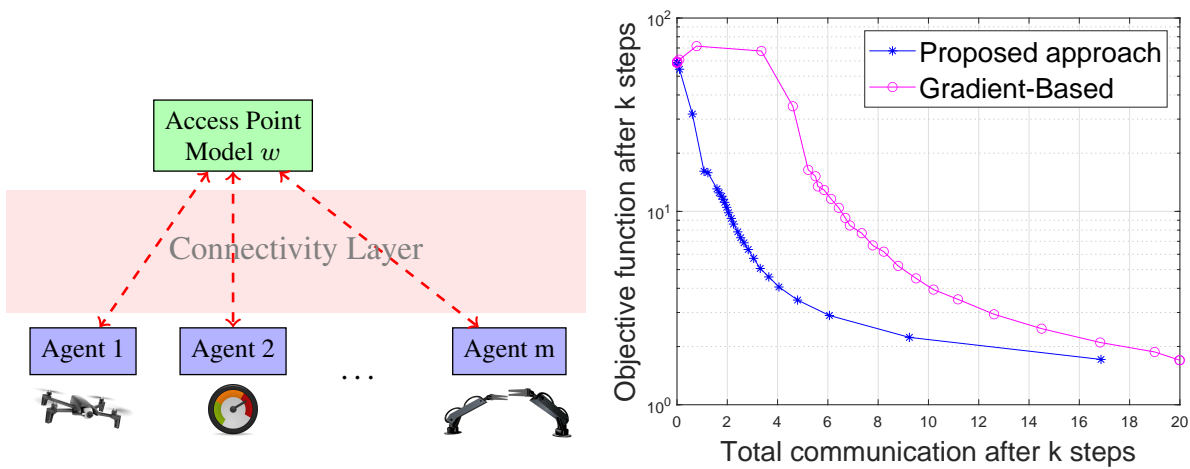}
	\caption{(Left) Architecture for solving machine learning tasks over networks of agents. Agents are collecting data and are communicating with an access point/server. 
		Examples include Industrial IoT systems with sensors/actuators/robots connected to common access points, collecting data to jointly update system models and application operating conditions.  
		{(Right) Comparison between our communication efficient learning approach based on estimating the gain in \eqref{eq:approximate_gain_1} versus the approach in \eqref{eq:scheduling_gradient} based on the magnitude of the gradients. 
		}
	}
	\label{fig:architecture}
	\vspace{-20pt}
\end{figure}


Specifically we consider the machine learning task to be a linear regression problem~\cite[Ch. 9]{shalev2014understanding}. 
We are interested in finding a vector of weights $w$ that explains the relationship between random variables $(x,y)\in \reals^n \times \reals$, i.e., $y \approx x^T w$. The random variables $(x,y)$ follow in general a joint distribution denoted by $\mu$. The desired choice for the weights is the one that minimizes the expected square prediction error, i.e.,
\begin{equation}\label{eq:regression_problem}
\min_{w} J(w) = \frac{1}{2} \mathbb{E}_{(x,y)\sim \mu}(y - x^T w)^2 
\end{equation}
where the expectation is with respect to the data distribution $\mu$ -- in the sequel we drop this notation when it is implied that expectation is with respect to this distribution.

The optimal solution $w^*$ is given as the solution to the linear equations 
\begin{equation}
\mathbb{E}xx^T w^* - \mathbb{E}xy = 0. 
\end{equation}
Towards finding an optimal set of weights, we would like to employ a gradient descent algorithm. Starting from some initial set of weights $w_0$ we would like to update the weights according to
\begin{equation}
w_{k+1} = w_k - \epsilon \nabla J(w_k)  
\end{equation}
where $\nabla J(w_k) = \mathbb{E}xx^T w_k - \mathbb{E}xy$, and $\epsilon >0$ is a small positive stepsize. As will be illustrated later, choosing $\epsilon<2/\lambda_{\max}(\mathbb{E}xx^T)$ guarantees convergence.

The distribution of the data is not a priori known, and hence as is common in machine leaning, e.g., in empirical risk minimization~\cite[Ch. 2]{shalev2014understanding}, we will attempt to minimize the empirical cost computed as an average over collected data. {Specifically we assume that at each iteration $k$ there are $N$ new data points 
	of the form 
	\begin{equation}\label{eq:data}
	(x_i, y_i) \in \reals^n \times \reals, \qquad i =1, \ldots, N.
	\end{equation} 
	We assume each data pair is independent and identically distributed according to a distribution $\mu$.\footnote{This setup arises either when an agent in Figure~\ref{fig:architecture} collects N new independent samples at each iteration, or when it just maintains a large pool of samples and selects randomly $N$ from them at each iteration as frequently done in stochastic gradient descent practice.}} 
Then we form the empirical cost
\begin{equation}
\hat{J}(w) = \frac{1}{2} \frac{1}{N} \sum_{i=1}^N (y_i - x_i^T w)^2
\end{equation}
%
%
With this approximation, 
we follow a stochastic gradient vector 
\begin{equation}
w_{k+1} = w_k  - \epsilon g_k
\end{equation}
computed over the data as
\begin{equation}\label{eq:gradient_estimate}
g_k= \nabla\hat{J}(w_k) = \frac{1}{N}\sum_{i=1}^N \left( x_i x_i^T w_k - x_i y_i\right)
\end{equation}
After this update the prediction error becomes
\begin{equation}
J(w_{k+1}) = \frac{1}{2} \mathbb{E}(y - x^T w_{k+1})^2
\end{equation}
where the expectation is with respect to the distribution $\mu$. We note that since the $N$ data points are random, so is the constructed gradient direction $g_k$, the updated vector $w_{k+1}$, as well as the performance metric $J(w_{k+1})$. 
{To evaluate how good is this updated prediction error, we would like to measure on average the quantity
	\begin{equation}
	\mathbb{E} [J(w_{k+1}) | w_k] = 
	\mathbb{E}_{data \sim \mu^N} [J(w_{k+1}) | w_k]
	\end{equation} 
	It is important to note here that the expectation is over the $N$ i.i.d. data that are collected at iteration $k$ and used to construct the stochastic gradient $g_k$. In the paper, whenever an expectation over iterates $w_k$ is taken, this is an expectation over the data collected until time $k$.}

\subsection{Communication-efficient learning problem}

Given the above modeling for a machine learning task that needs to be solved, the communication problem is as follows. {At each iteration $k$, the server broadcasts the current weights $w_k$ to all agents.} Then each agent $i$ collects $N$ local data points identically distributed (across time and across agents), computes a local stochastic gradient $g_k^i$ from the available local data, and decides whether to transmit this gradient update over the communication network to the receiving server. The server maintains a current vector of weights $w_k$ which will be updated depending on the information received from different agents. For simplicity of exposition the case of two agents is considered, leading to the update rule at the server
\begin{equation}\label{eq:dynamics_single_task}
w_{k+1} = \left\{ \begin{array}{ll} w_k - \epsilon g_k^1 &\text{if agent $1$ transmits}\\
w_k - \epsilon g_k^2 &\text{if agent $2$ transmits}\\
w_k - \epsilon/2 ( g_k^1 +g_k^2) &\text{if both agents transmit}\\
w_k &\text{if no agent transmits} \end{array}\right.
\end{equation}
We also denote with $\alpha_k^i \in \{1,0\}$ the decision for each agent $i$ to transmit or not. 
{At the next iteration $k+1$ a new set of data is collected as in \eqref{eq:data} at each agent, a new stochastic gradient direction $g_{k+1}^i$ is computed at each agent, and the process repeats. The aim will be to\textit{ avoid sending updates all the time in order to limit the communication burden}.}

\begin{remark}[Scope of the setup]
	The setup (linear regression, two agents) is chosen as a basis for theoretical joint analysis of convergence and communication utilization, illustrating inefficiencies of approaches in the literature (Remark~\ref{rem:comparisons}). Besides, linear regression forms the basis for relevant problems in the control systems and learning community, and extensions to general convex problems and more agents is under investigation. 
\end{remark}

\section{Proposed communication-efficient learning}\label{sec:approach}

{ The approach is based on the notion of performance gain which can be thought as a measure of how informative are the data collected at each agent at each time step with respect to the machine learning problem. The gain at agent $i=1,2$ can be calculated by
	measuring how much will the objective change if the agent sends the update
	Whether this gain is negative or positive depends on the random direction of the update. 
	The proposed approach then is to send a gradient update if the gain is large enough. Mathematically we write
	\begin{equation}\label{eq:single_scheduling}
	\alpha_k^i = \left\{ \begin{array}{ll} 1 &\text{if } J(w_k - \epsilon g_k^i) - J(w_k) \leq -\lambda \\
	0 &\text{otherwise}\end{array}\right.
	\end{equation}
	for some scalar parameter $\lambda>0$.
	Intuitively this approach saves up communication resources, because the larger the parameter value $\lambda$ is, the more infrequent the updates will be. But then the question is what can be said about the progress of learning. We have then the following result.
	
	\begin{thm}[Convergence]\label{thm:theorem_single}
		Consider the optimization problem defined in \eqref{eq:regression_problem} and let $w^*$ be the optimal solution. Consider the update rule in \eqref{eq:dynamics_single_task}. Suppose $g_k^i, i=1,2,$ are independent random variables with mean equal to $\nabla J(w_k)$ and covariance $G$ at each iteration $k$. Consider the communication strategy in \eqref{eq:single_scheduling}. Then for any iteration $N$ we have that 
		\begin{equation}\label{eq:single_statement}
		\mathbb{E}J(w_N) \leq \rho^N J(w_0) + 
		(1-\rho^N) \left[ J(w^*) + \frac{\epsilon^2 \text{Tr}(\Sigma_x  G)}{1-\rho} \right] 
		+ \lambda\sum_{\ell=0}^N \rho^{N-\ell} \frac{\sum_{i=1}^2 \mathbb{E}(1-\alpha_\ell^i)}{2}
		\end{equation}
		where the expectation is with respect to the data collected until iteration $N$, and the parameters  are $\Sigma_x = \mathbb{E}xx^T/2$ and $\rho = \max_i (1-\epsilon  \lambda_i(\mathbb{E}xx^T))^2$ and the stepsize $\epsilon>0$ is chosen small enough so that $\rho<1$. 
	\end{thm}
	
	\begin{proof}
		Note that by the dynamics in \eqref{eq:dynamics_single_task} we can write
		\begin{align}\label{eq:four_cases}
		J(w_{k+1}) &= (1- \alpha_k^1)(1- \alpha_k^2) J(w_k) + \alpha_k^1 (1- \alpha_k^2) J(w_k - \epsilon g_k^1)
		\notag \\ &+ (1- \alpha_k^1)\alpha_k^2  J(w_k - \epsilon g_k^2)
		 + \alpha_k^1 \alpha_k^2 J(w_k - \epsilon/2 g_k^1- \epsilon/2 g_k^2), 
		\end{align}
		depending on each of the four cases.
		Then due to the convexity of the problem we have for the last case the bound 
		\begin{equation}
			J(w_k - \epsilon/2 g_k^1- \epsilon/2 g_k^2) \leq 1/2 J(w_k - \epsilon g_k^1) + 1/2 J(w_k - \epsilon g_k^2).
		\end{equation}
		Substituting this bound in \eqref{eq:four_cases} and after a  rearrangement of terms we get
		\begin{align}\label{eq:more_cases}
		J(w_{k+1}) &\leq \frac{1}{2} (1- \alpha_k^2) \Big[ (1- \alpha_k^1)J(w_k) + \alpha_k^1 J(w_k - \epsilon g_k^1) \Big]
		+ \frac{1}{2}\alpha_k^1 J(w_k - \epsilon g_k^1) 
		\notag \\ &+ \frac{1}{2} (1- \alpha_k^1) \Big[ (1- \alpha_k^2)J(w_k) + \alpha_k^2 J(w_k - \epsilon g_k^2) \Big]
		+ \frac{1}{2} \alpha_k^2 J(w_k - \epsilon g_k^1)
		\end{align}
		Then the terms in the brackets can be bounded. Note that due to the choice in \eqref{eq:single_scheduling} the following inequality holds for all times (technically it holds almost surely as all the variables involved are random variables)
		\begin{equation}
		(1-\alpha_k^i) J(w_k) + \alpha_k^i J (w_k - \epsilon g_k^i) \leq \lambda + J(w_k - \epsilon g_k^i).
		\end{equation}
		This can be easily verified by examining the two cases $\alpha_k^i =0$ or $1$ separately.
		Substituting this inequality for agents $i=1,2$ in \eqref{eq:more_cases} we get
		\begin{align}\label{eq:bound_more_cases}
		J(w_{k+1}) &\leq \frac{1}{2} (1- \alpha_k^2) \Big[ \lambda + J(w_k - \epsilon g_k^1) \Big]
		+ \frac{1}{2}\alpha_k^1 J(w_k - \epsilon g_k^1)
		\notag \\ &+ \frac{1}{2} (1- \alpha_k^1) \Big[ \lambda + J(w_k - \epsilon g_k^2) \Big]
		+ \frac{1}{2} \alpha_k^2 J(w_k - \epsilon g_k^2)
		\end{align}
		Taking expectation over the stochastic gradients $g_k^1$ and $g_k^2$, conditioned on the current iterate $w_k$, and using the symmetry of the problem with respect to agents $i=1,2$ we get that 
		\begin{align}\label{eq:intermediate}
		\mathbb{E}[ J(w_{k+1}) \given w_k] &\leq 
		\mathbb{E}[1 - \alpha_k^i\given w_k] \Big[ \lambda + \mathbb{E}[J(w_k - \epsilon g_k^i)\given w_k] \Big]+ \mathbb{E}[\alpha_k^i J(w_k - \epsilon g_k^i)\given w_k]
		\end{align}
		Then we have the following key fact, which \textit{is shown separately in the Appendix, }
		\begin{equation}\label{eq:main_comparison}
			\mathbb{E}[\alpha_k^i J(w_k - \epsilon g_k^i)\given w_k] \leq \mathbb{E}[\alpha_k^i\given w_k] \;  \mathbb{E}[J(w_k - \epsilon g_k^i)\given w_k]
		\end{equation}
		Substituting this bound in \eqref{eq:intermediate}, we get
		\begin{align}
		\mathbb{E}[ J(w_{k+1}) \given w_k] &\leq 
		 \mathbb{E}[1 - \alpha_k^i\given w_k]  \lambda + \mathbb{E}[J(w_k - \epsilon g_k^i)\given w_k] 
		\end{align}
		Then given the fact that the function $J(w)$ is quadratic, and the property of the stochastic gradient that the mean is unbiased $\mathbb{E}g_k^i= \nabla_{w_k} J(w_k)$ with a constant variance, we get that\footnote{We exploit the fact that $(I-\epsilon2\Sigma_x)'\Sigma_x (I-\epsilon2\Sigma_x)\preceq \rho \Sigma_x$ }
		\begin{align}
		\mathbb{E}[ J(w_k - \epsilon g_k)&\given w_k] \leq \rho J(w_k) + \epsilon^2 \text{Tr}(\Sigma_x  G) + (1-\rho)J(w^*)
		\end{align}
		Substituting this we get,
		\begin{align}\label{eq:intermedient}
		\mathbb{E}[ J(w_{k+1}) \given w_k] &\leq 
		\mathbb{E}[1- \alpha_k^i\given w_k]  \lambda + \rho J(w_k) + \epsilon^2 \text{Tr}(\Sigma_x  G) + (1-\rho)J(w^*)
		\end{align}
Taking expectation on both sides with respect to the variable $w_k$, and iterating over time $k=1, \ldots, N$, we get the desired result \eqref{eq:single_statement}.
	\end{proof}

	The result verifies that the update rule converges (in a stochastic sense) because $\rho<1$ as can be confirmed by the appropriate choice of the stepsize $0<\epsilon<2/\lambda_{\max}(\mathbb{E}xx^T)$. Essentially the result follows because the function $J(w)$ can be thought as a Lyapunov function for the stochastic dynamics of the update in \eqref{eq:dynamics_single_task}. A direct consequence of the above result is 
	\begin{equation}
	\limsup_{N \rightarrow \infty} \, \mathbb{E}J(w_N) \leq  J(w^*) + \frac{\lambda + \epsilon^2 \text{Tr}(\Sigma_x  G)}{1-\rho} 
	\end{equation}
	This means that eventually we get  close to the optimal set of weights $w^*$ subject to some overshoots. The latter are due to the stochastic gradient and its covariance $G$, which can be made small in practice by choosing the step size $\epsilon$ to be small -- or by choosing a diminishing stepsize which will be analyzed in future work. Moreover, there is a penalty proportional to the parameter $\lambda$, introduced to save up on communication cost.  It is also possible to choose a diminishing parameter $\lambda$ to eliminate this effect.

	\begin{remark}
		In Theorem~\ref{thm:theorem_single} we assumed for simplicity that the stochastic gradients have bounded covariances that are constant over time. In reality for the problem above the covariance of the stochastic gradient in \eqref{eq:gradient_estimate} will depend on the current iterate $w_k$, but our choice can be justified in two ways. We can either consider these covariances to be uniformly bounded over time by some constant $G$. Or alternative if we consider the case close enough to the equilibrium $w_k \approx w^*$, then it follows that the covariances are indeed constant over time. A more detailed investigation will be explored in a follow up work.
	\end{remark}
	
	{ Furthermore, we can establish the following guarantee about the total communication rate of the proposed approach. 
		
		\begin{thm}[Communication guarantee]
			Consider the same setup as in Theorem~\ref{thm:theorem_single}.
			The total communication rate satisfies 
			\begin{equation}\label{eq:communication_guarantee}
			\limsup_{N \rightarrow \infty} \sum_{k=0}^N  \max\{\alpha_k^1,\alpha_k^2\}
			\leq\frac{ J(w_0) - J(w^*)}{\lambda }
			\end{equation}
			almost surely, 
			with respect to the data collected as iterations $N\rightarrow \infty$.
		\end{thm}
	}
	\begin{proof}
		{ 
			Due to the choice in \eqref{eq:single_scheduling} the following inequality holds for all times (technically it holds almost surely as all the variables involved are random variables)
			\begin{equation}\label{eq:basic_inequality_2}
			\lambda \max\{\alpha_k^1,\alpha_k^2\} + 
			J(w_{k+1}) \leq J(w_k).
			\end{equation}
			This can be easily verified by examining the four cases for $\alpha_k^i =0$ or $1$ for $i=1,2$. Specifically, when both $\alpha_k^1 = \alpha_k^2 = 1$ we have that
			\begin{align}
				J(w_{k+1}) &=J(w_k - \epsilon/2 g_k^1- \epsilon/2 g_k^2) \leq 1/2 J(w_k - \epsilon g_k^1) + 1/2 J(w_k - \epsilon g_k^2) \notag\\
				&\leq 1/2 (J(w_k) - \lambda) +1/2 (J(w_k) - \lambda).
			\end{align}
			where the first inequality holds due to convexity and the second inequality holds due to the choice in \eqref{eq:single_scheduling}.
			
			Iterating \eqref{eq:basic_inequality_2} over time $k=0, \ldots, N$, and summing up, we conclude that
			\begin{equation}
			\lambda \sum_{k=0}^N   \max\{\alpha_k^1,\alpha_k^2\} + 
			J(w_{N+1}) \leq J(w_0).
			\end{equation}
			Moreover, since any value of the variable $w_{N+1}$ is in general suboptimal, we have that $J(w_{N+1})\geq J(w^*)$.
From which we get the desired result \eqref{eq:communication_guarantee}. 
		}
	\end{proof}
	
	This result counts communication as long as one agent transmits. It guarantees explicitly that increasing $\lambda$ will decrease the resulting communication in an inversely proportional manner.
	
	\subsection{Practical communication scheme}
	
	Despite the above guarantee, implementing the proposed communication scheme in \eqref{eq:single_scheduling} would be practically challenging because it requires information that is not known. Specifically it would require knowledge of the data distribution in order to compute the actual performance gain. Since the true distribution is unknown, one approach is to \textit{estimate the performance gain from the data}. In particular, since the objective function is quadratic, we can write the performance gain as
	\begin{align}\label{eq:gain_1}
	J(w_{k}- \epsilon g_k) - J(w_k) = &-\epsilon g_k^T \nabla J(w_k) + \frac{1}{2} \epsilon^2 g_k^T \nabla^2 J(w_k) g_k
	\end{align}
	This is a quadratic function of the stochastic gradient $g_k$. Then we can approximate the quantities
	\begin{align}
	&\nabla J(w_k) \approx \frac{1}{N}\sum_{i=1}^N \left( x_i x_i^T w_k - x_i y_i\right) =g_k, 
	\qquad \nabla^2 J(w_k) \approx \frac{1}{N}\sum_{i=1}^N x_i x_i^T
	\end{align}
	where we note that the stochastic gradient direction $g_k$ appears again. Hence, using the expression for the information gain in \eqref{eq:gain_1}, we can approximate the gain as\footnote{Overall at each agent these require $O(Nn)$ operations hence are scalable.}
	%
	\begin{align}\label{eq:approximate_gain_1}
	J(w_{k}- \epsilon g_k) - J(w_k) \approx -\epsilon g_k^T\left[ I - \epsilon \frac{1}{2}  \frac{1}{N}\sum_{i=1}^N x_i x_i^T \right] g_k
	\end{align}
	It is crucial to emphasize that \textit{this is no longer a simple quadratic function} of the data but a more complicated function - we note that the data appear both in the stochastic gradients $g_k$ by \eqref{eq:gradient_estimate} as well as the matrix in the middle. This approximate value of the gain may take again positive or negative values but it induces an approximation error/bias.
	
	As a result, we can implement the communication decision in \eqref{eq:single_scheduling} with the approximation in \eqref{eq:approximate_gain_1}. In this case we no longer have the performance guarantee in Theorem~\ref{thm:theorem_single}. In numerical evaluations however we see that despite the bias this mechanism performs very well.
	

}


\begin{remark}[Other approaches in the literature]\label{rem:comparisons}
A different perspective would be to treat the agents with the largest updates as the most important, and let an agent communicate if the norm of its (stochastic) gradient is large, i.e., 
\begin{equation}\label{eq:scheduling_gradient}
\alpha_k^i = \left\{ \begin{array}{ll} 1 &\text{if } \|g_k^i\|^2 \geq \mu \\
0 &\text{otherwise}\end{array}\right.
\end{equation}
for some scalar parameter $\mu>0$. From our expression on \eqref{eq:approximate_gain_1} we see that for small stepsizes $\epsilon$ the magnitude of the gradient may serve as a proxy for the performance gain.
But in numerical comparisons we show that this scheme typically leads to worse performance.  This may also be the case when the Hessian of the problem is further from an identity matrix. The idea of scheduling based on gradient magnitudes has been proposed in very recent works in federated learning over wireless channels~\citep{amiri2020update, chen2020convergence}, and in the context of sparsification and quantization for high-dimensional gradient updates~\citep{aji2017sparse, sattler2019sparse}. Our findings hence point to a \textit{novel and more communication- efficient approach for gradient updates}.
{Finally, a different perspective is followed by \citep{chen2018lag,chen2018communication}; when agents do not update their gradients at the server, the server just keeps a memory of past received gradients and uses them for gradient descent. A difference compared to the present paper is that here there is an explicit communication-learning tradeoff controlled by the parameter $\lambda$. A more detailed comparison between that approach and the one in the present paper will be considered in future work.}
\end{remark}

\section{Numerical results}\label{sec:numerical}

In this section we make an additional assumption about the data samples, that $x_i$ are i.i.d. Gaussian random variables, while the points $y_i$ are given as $y_i = x_i^T w^* + \eta_i$ where $w^*$ is the true parameter and $\eta_i$ are i.i.d. Gaussian measurement noises. These assumptions are not necessary for the theoretical analysis above.


\begin{figure}
	\includegraphics[width=0.5\columnwidth]{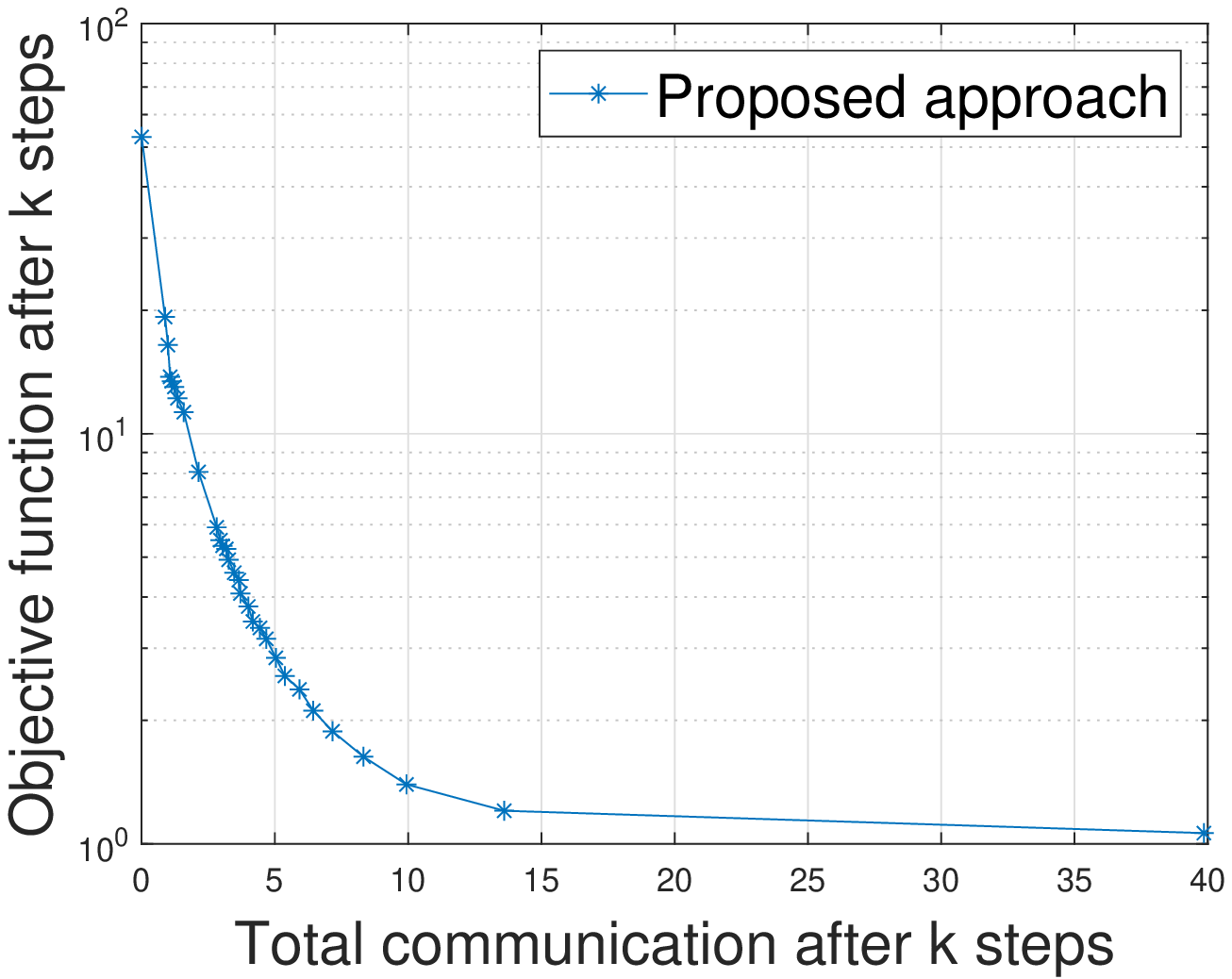}
	\includegraphics[width=0.5\columnwidth]{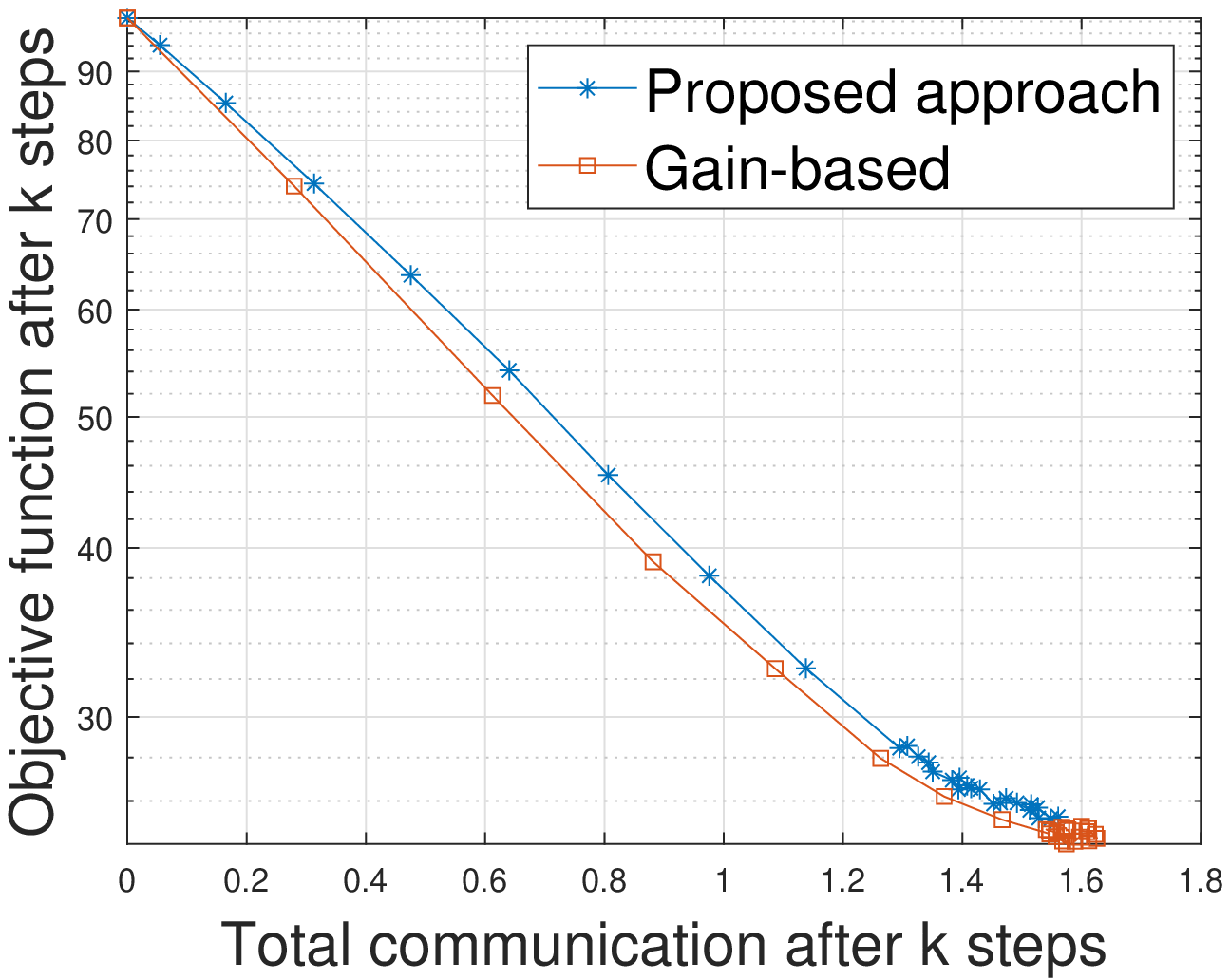}
	\caption{(Left) Evaluation of the tradeoff between communication rate and machine learning performance of the proposed  algorithm in \eqref{eq:single_scheduling}. 
	{(Right) Comparison between our communication approach in \eqref{eq:single_scheduling} requiring the data distribution to compute the gains by \eqref{eq:gain_1} versus estimating the gain by \eqref{eq:approximate_gain_1}. }
}
	\label{fig:tradeoff}
	\vspace{-20pt}
\end{figure}

{We consider the communication algorithm in \eqref{eq:single_scheduling} with the performance gains estimated as in \eqref{eq:approximate_gain_1}. We consider $m=2$ agents. We consider a problem with dimensions $n=2$, with covariances $\mathbb{E} x x^T = \left[ \begin{array}{cc} 3 &0\\ 0 &1 \end{array}\right]$ (which affects the Hessian of the problem), the initial weights are $w_0=0$, and the true weights equal to $w^* = \left[ \begin{array}{c} 3 \\ 5 \end{array}\right]$. 
First for stepsize $\epsilon = 0.1$ and $N=5$ data points available at each iteration and at each agent (cf.\eqref{eq:data}), we simulate algorithm \eqref{eq:single_scheduling} for varying values of the parameter $\lambda$. In Fig.~\ref{fig:tradeoff}(Left) we plot the observed mean learning performance after the $K=10$ iterations ($J(w_{K})$) versus the total communication rate ($\sum_{k=0}^{K} \sum_{i=1}^2\alpha_k^i$). We observe that the proposed communication approach indeed allows us to tradeoff communication rate with machine learning performance.
	
	
}


We would like to investigate how much bias is introduced by our practical scheme that is based on estimating the performance gain at each agent based on the currently available data. Hence we compare \eqref{eq:single_scheduling} when using the performance gains computed by \eqref{eq:gain_1} that requires knowledge of the data distributions, with the completely data-based scheme in \eqref{eq:approximate_gain_1}. 
For the same linear regression setup as before, for $N=5$ samples per agent, stepsize $\epsilon = 0.2$ and for a single time step,  for varying value of the parameter $\lambda$ the comparison is shown in Fig.~\ref{fig:tradeoff}(Right).
In our numerical evaluations, we surprisingly do not observe a significant difference due to the estimation procedure. This was observed across different instances, reinforcing the usefulness of our scheme. 




We finally compare our communication scheme \eqref{eq:single_scheduling} based on estimating the performance gain across tasks in \eqref{eq:approximate_gain_1} with the simple strategy based on the magnitude of the gradients at each agent in \eqref{eq:scheduling_gradient}. 
{We consider  a randomly chosen $w^*$ of dimensions $n=10$ and a covariance matrix $\mathbb{E} x x^T$ diagonal with randomly chosen coefficients.}
We assume $N=20$ data points are available at each iteration per agent. We consider $K=10$ steps in the algorithm. 
For stepsize $\epsilon = 0.2$ the comparisons are shown in Fig.~\ref{fig:architecture}(Right) 
for varying values of the parameters $\lambda$ and $\mu$ in each of the schemes.
%
%
We observe that our approach performs significantly better than the gradient-based one. The improvements get typically more significant as the setpsize increases. Our conclusion is that \textit{the magnitude of the gradient is not a reliable measure for the informativeness of the data}. Our approach which is based on the more complex estimate of performance gain provides a more reliable and communication-efficient approach.

\section{Concluding remarks}

In this paper we examine the problem of solving machine learning tasks over a network. We consider the problem of selecting which updates to communicate to lower the communication rate. To exploit the informativeness of the data we examine the notion of performance gain and we illustrate numerically how this can be approximated from the data without further model knowledge. The approach is contrasted to other related works in the area of communication-efficient learning. 
Ongoing work explores the use of the approach in more complex networks of learning agents, as well as other machine learning tasks beyond linear regression.

\appendix
\section{Technical Results}

\textbf{Proof of \eqref{eq:main_comparison}}. 
Let us consider the distribution of the gradient $g$ denoted by $F(g)$. Then we can rewrite \eqref{eq:main_comparison} as 
\begin{equation}
\int \alpha(g)J(w-\epsilon g) dF(g) \leq \int \alpha(g)dF(g) \int J(w-\epsilon g) dF(g)
\end{equation}
However, by definition of the communication rule \eqref{eq:single_scheduling} we have that $\alpha(g)=1$ only when $J(w-\epsilon g) \leq J(w)-\lambda$ and zero otherwise. Let us define this set of values $S= \{g\in \reals^n: J(w-\epsilon g) \leq J(w)-\lambda\}$. Then \eqref{eq:main_comparison} is equivalent to
\begin{equation}
\int_S J(w-\epsilon g) dF(g) \leq \int_S dF(g) \left[ \int_S J(w-\epsilon g) dF(g) + \int_{S^c} J(w-\epsilon g) dF(g)\right]
\end{equation}
%
%
%
which is 
equivalent to
\begin{align}\label{eq:equivalent_inequality}
\int_{S^c} dF(g) \int_S J(w-\epsilon g) dF(g) \leq \int_S dF(g) \int_{S^c} J(w-\epsilon g) dF(g)
\end{align}
We can bound the left hand side because we can bound $J(w-\epsilon g)$ point wise on the set $S$ as
\begin{align}\label{eq:bound_1}
\int_{S^c} dF(g) \int_S J(w-\epsilon g) dF(g) \leq \int_{S^c} dF(g) \int_S dF(g) \; [J(w)-\lambda]
\end{align}
Further we can bound the right hand side of \eqref{eq:equivalent_inequality} as
\begin{align}\label{eq:bound_2}
\int_S dF(g) \int_{S^c} J(w-\epsilon g) dF(g) \geq \int_S dF(g) \int_{S^c} dF(g) \; [J(w)-\lambda]
\end{align}
Combining \eqref{eq:bound_1} and \eqref{eq:bound_2} we verify \eqref{eq:equivalent_inequality} and conclude the proof.

\pagebreak
\bibliography{federated_learning}

\end{document}